\documentclass[acmtog]{acmart}
\acmSubmissionID{336}

\usepackage{booktabs} %

\usepackage[ruled]{algorithm2e} %
\usepackage{thmtools}
\usepackage{thm-restate}
\usepackage{listings}
\usepackage[capitalize]{cleveref}
\crefname{section}{Sec.}{Secs.}
\Crefname{section}{Section}{Sections}
\Crefname{table}{Table}{Tables}
\crefname{table}{Tab.}{Tabs.}
\Crefname{figure}{Figure}{Figures}
\crefname{figure}{Fig.}{Figs.}

\usepackage{siunitx}
\usepackage{multirow,tabularx}

\SetAlFnt{\small}
\SetAlCapFnt{\small}
\SetAlCapNameFnt{\small}
\SetAlCapHSkip{0pt}

\setcopyright{rightsretained}
\acmJournal{TOG}
\acmYear{2024}
\acmVolume{43}
\acmNumber{4}
\acmArticle{107}
\acmMonth{7}
\acmDOI{10.1145/3658156}

\begin{document}
\title{3Doodle: Compact Abstraction of Objects with 3D Strokes}

\author{Changwoon Choi}
\orcid{0000-0001-5748-6003}
\affiliation{%
  \institution{Seoul National University (ECE)}
  \country{South Korea}}
\email{changwoon.choi00@gmail.com}
\author{Jaeah Lee}
\orcid{0009-0004-2648-8523}
\affiliation{%
  \institution{Seoul National University (IPAI)}
  \country{South Korea}
}
\email{hayanz@snu.ac.kr}
\author{Jaesik Park}
\orcid{0000-0001-5541-409X}
\affiliation{
    \institution{Seoul National University (CSE, IPAI)}
    \country{South Korea}
}
\email{jaesik.park@snu.ac.kr}
\author{Young Min Kim}
\orcid{0000-0002-6735-8539}
\affiliation{
    \institution{Seoul National University (ECE, IPAI)}
    \country{South Korea}
}
\email{youngmin.kim@snu.ac.kr}

\begin{abstract}
While free-hand sketching has long served as an efficient representation to convey characteristics of an object, they are often subjective, deviating significantly from realistic representations. Moreover, sketches are not consistent for arbitrary viewpoints, making it hard to catch 3D shapes.
We propose 3Dooole, generating descriptive and view-consistent sketch images given multi-view images of the target object. Our method is based on the idea that a set of 3D strokes can efficiently represent 3D structural information and render view-consistent 2D sketches. We express 2D sketches as a union of view-independent and view-dependent components. 3D cubic B\'ezier curves indicate view-independent 3D feature lines, while contours of superquadrics express a smooth outline of the volume of varying viewpoints.
Our pipeline directly optimizes the parameters of 3D stroke primitives to minimize perceptual losses in a fully differentiable manner. 
The resulting sparse set of 3D strokes can be rendered as abstract sketches containing essential 3D characteristic shapes of various objects. We demonstrate that 3Doodle can faithfully express concepts of the original images compared with recent sketch generation approaches.\footnote{Code available at \url{https://github.com/changwoonchoi/3Doodle}}
\end{abstract}




\begin{teaserfigure}
    \centering
    \includegraphics[trim={5mm 50mm 5mm 75mm}, clip, width=0.95\textwidth]{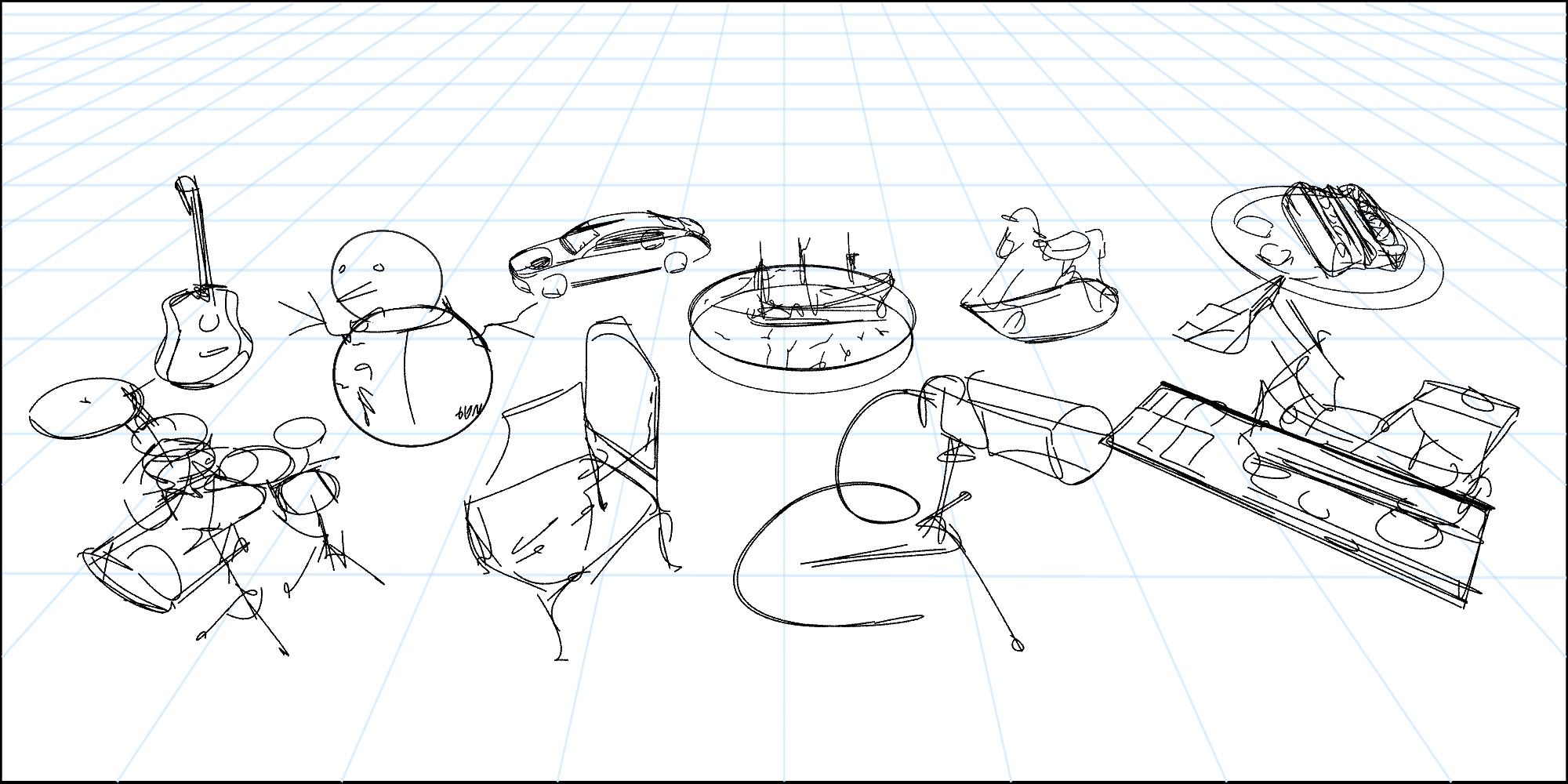}
    \caption{We propose 3Doodle, a method to draw 3D strokes from multi-view images. Our approach can express the target object with a highly compact set of geometric primitives, and the recovered 3D strokes are rendered to generate view-coherent sketches.
    }
    \label{fig:teaser}
\end{teaserfigure}

\maketitle

\section{Introduction}

Free-hand sketching is an effective tool for visual communication, as only a few sketch lines can summarize perceptually distinctive components. 
Sketch strokes are also intuitive tools to interact with users in editing an image in many graphics applications~\cite{yu2019free, zeng2022sketchedit}.
Previous works further explore the connection between 3D structure and conventional sketch lines and show promising results in converting sketches into 3D objects~\cite{Li:2020:Sketch2CAD,li2022free2cad,guillard2021sketch2mesh,zhang2021sketch2model} or extracting sketch lines from detailed 3D geometry~\cite{liu2020neural,benard2019line,liu2021neural}. 
Like image editing applications, user sketches provide an interactive tool to edit the 3D models~\cite{mikaeili2023sked,bandyopadhyay2023doodle}.

Sketch primitives with 3D context can be an effective communication tool for both structural and semantic context, widening the scope of potential applications for VR/AR scenarios.
The aforementioned works assume full 3D models are available for generating sketches or editing 3D objects with line strokes. However, it is not trivial to formulate the exact mapping from artistic sketch curves on-screen space to the underlying 3D structure of the object.
While sketches may contain important spatial arrangements of the object, hand-drawn sketches do not strictly abide by the 3D geometry, either due to a lack of delicate skills or the artist's creative expression.
Therefore, we cannot extract consistent 3D lines from 2D sketch lines extracted individually from different viewpoints. 
Previous works also propose sophisticated techniques for converting 3D models, such as high-resolution triangular mesh, into sketch images, which result in delicate renderings rather than sparse and abstract strokes~\cite{decarlo2003suggestive,judd2007apparent}.
The task does not rely on a data-driven method, as it is hard to obtain a dataset of 3D object-sketch pairs in sufficient quantity and quality.

We suggest that a compact set of 3D geometric primitives can efficiently represent 3D structural information, rendering expressive 2D sketch lines from different viewpoints. 
Specifically, we assume sketch lines are composed of view-independent or view-dependent lines. 
View-independent lines are those from texture or sharp edges, which are distinctive feature lines of the underlying 3D geometry of the given object. 
Our 3D cubic B\'ezier curves are expected to locate the view-independent feature lines. 
View-dependent lines, conversely, come from contours of smooth surface boundaries, whose exact 3D location changes by viewpoints. 
We assign unions of superquadrics to encompass the complex surface contours and provide a differentiable formulation that converts their surfaces into sketch lines for the given viewpoints.

We directly optimize for proposed lightweight 3D primitives given only multi-view images of an object as input. 
We do not need to individually stylize the input images into sketches, and we also do not require any explicit 3D mesh or NeRF volume~\cite{mildenhall2021nerf}. 
Most importantly, our rendering pipeline is fully differentiable for a small set of parameters representing the B\'ezier curves and superquadrics. 
We can then directly optimize the parameters to minimize the popular perceptual loss functions, such as CLIP~\cite{radford2021learning} and LPIPS~\cite{zhang2018unreasonable}, so that the rendered images are perceived as hand-drawn 2D sketches from various viewpoints. 
Further, the coherent 3D structure summarizes important geometric features that satisfy multi-view consistency, which can potentially be extended to various 3D manipulation tasks.

In summary, our key contributions are highlighted as follows:
\begin{itemize}
    \item We propose 3Doodle, a first approach to generate expressive sketches from a set of 3D stroke primitives that summarizes important semantic information from image observations.
    \item Our approach can directly find the 3D strokes from multi-view images without data-driven training, stylizing input images, or reconstructing 3D models such as NeRF or mesh.
    \item We suggest highly compact 3D stroke representation (less than 1.5kB) of view-independent and view-dependent components, which together can quickly draw conceptual sketch lines of various shapes of objects.
    \item We introduce a fully differentiable rendering method, effectively optimizing powerful perceptual losses with a small set of parameters.
\end{itemize}
We demonstrate that the proposed compact 3D representation can easily mimic an expressive hand drawing of diverse pictured objects, including real-world objects. Our 3Doodle faithfully expresses concepts of the target shape and provides view-consistent sketch images.

\section{Related Works}
\paragraph{Sketch Generation from Image}
Generating sketches from photos of objects suffers from a large domain gap between pictures and sketches. 
Images contain dense, detailed evidence of real objects. In contrast, sketches are sparse and abstract, as they are traditionally a quick tool for an artist to summarize the essential components of the given object on an image plane. 
While extracting edges from image pixels~\cite{canny1986computational, su2021pixel} can also create a line drawing of the object, they do not enjoy the compactness and implicit semantic mapping of conventional sketches.
Numerous works overcome the domain gap by exploiting datasets composed of pairs of images and sketches~\cite{song2018learning, kampelmuhler2020synthesizing}. 
However, obtaining high-quality sketches of given images on a large scale is difficult.
Most sketch generation methods only cover specific categories in datasets, such as human portrait~\cite{berger2013style, yi2019apdrawinggan}, or a limited number of object categories ~\cite{yu2017sketchx, eitz2012hdhso, ha2017neural}.
A few works alleviate the necessity of paired datasets and generate sketches of arbitrary objects with style transfer ~\cite{huang2017arbitrary, liu2021deep} or image translation~\cite{zhu2017unpaired, Isola_2017_CVPR}.
Although they can mimic the style of the free-hand sketches to some extent, the generated pixels are still composed of dense information from original images.
The images in the converted style do not necessarily maintain compact and sparse representation, which are among the most powerful features of sketches.

3Doodle proposes maintaining the sparse characteristics by directly optimizing a set of strokes, similar to recent approaches~\cite{Li:2020:DVG, vinker2022clipasso}.
The direct optimization does not require any dataset to train and, therefore, can generate sketches of arbitrary class objects.
However, previous works optimize 2D strokes in the input image view without explicit 3D structure.
Instead, we optimize 3D strokes with a novel differentiable rendering pipeline.
3Doodle, therefore, generates a compact 3D representation, which can be rendered in various viewpoints to deliver abstract semantic structure.

\paragraph{3D Non-Photorealistic Rendering}
Non-photorealistic rendering (NPR) renders 3D representations that exhibit various artistic styles. 
Among possible variations, our sketch generation task particularly relates to line drawing from 3D models.
Line drawings reflect how humans understand the world~\cite{benard2019line}.
Previous works on line-drawing extraction trace feature lines on surface extrema~\cite{ohtake2004ridge} in addition to view-dependent geometric boundaries such as occluding contours~\cite{benard2019line} or suggestive contours~\cite{decarlo2003suggestive}. 
A more recent method~\cite{liu2020neural} combines an entirely geometric approach with image translation and trains a neural network that generates sketches from 3D models.
However, all the works rely on having an accurate mesh model, which may not be accessible in real-world scenarios.
Further, geometric processing results deviate from typical free-hand drawings because they ignore the effects of texture variations, and lines are densely populated compared to rough sketches.

Recent NPR approaches propose relaxing the mesh requirement and using the 3D volume of neural radiance fields (NeRF) trained from multi-view images instead.
They transform pre-trained NeRF into the style of an exemplar image to generate multi-view consistent sketches~\cite{zhang2022arf, Zhang_2023_CVPR}.
While the pipelines use more casual input than a complete mesh, they need heavy computation first to train a neural volume.
Neural Edge Fields (NEF)~\cite{Ye_2023_CVPR} also extracts a compact set of parametric curves as 3Doodle, but it is converted with additional 2D edge input from pre-trained neural fields.
Furthermore, the 3D lines from NEF cannot represent view-dependent contours, which are crucial in expressing the overall shape of the object.

\section{3Doodle}
\label{sec:Sketch3D}

\begin{figure*}
    \centering
    \includegraphics[width=0.87\linewidth]{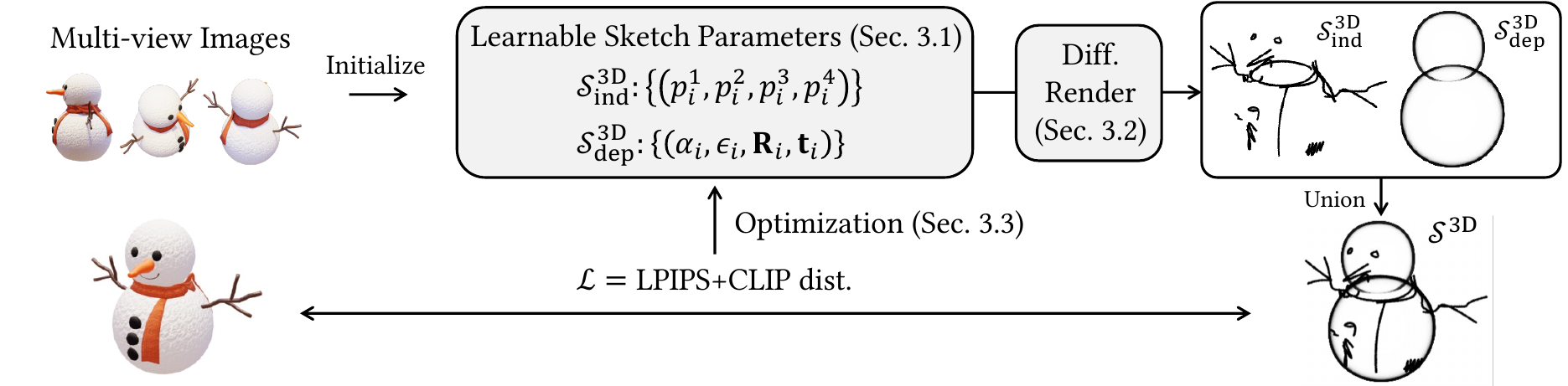}
    \caption{
        Overview of 3Doodle.
        We generate a compact 3D geometric representation from multi-view images of objects.
        We separately define view-independent stroke ($\mathcal{S}^{\text{3D}}_{\text{ind}}$) and view-dependent stroke ($\mathcal{S}^{\text{3D}}_{\text{dep}}$).
        We represent a view-independent stroke as a set of 3D B\'ezier curves and a view-dependent stroke as a contour of superquadrics. (Sec. 3.1)
        We also propose a fully differentiable rendering method to render sketches from the 3D strokes. (Sec. 3.2)
        Finally, our 3D stroke parameters are directly optimized with perceptual losses. (Sec. 3.3)
    }
    \label{fig:method_overview}
\end{figure*}

We first introduce our 3D geometric primitives that generate sketch images in~\cref{subsec:representation}. 
Then our differentiable rendering pipeline in~\cref{subsec:diff_rendering} converts the proposed primitives into an abstract sketch image.
The compact set of parameters is effectively optimized to minimize the perceptual loss function as described in~\cref{subsec:optimization}.

\subsection{Sketch Representation}
\label{subsec:representation}

We propose a coherent set of 3D strokes which can be rendered into 2D sketch lines. 
Parts of sketch lines can be rendered from 3D geometric feature lines, such as ridges and valleys of the surface or texture edges.
Such lines can be represented directly from curves that reside in 3D space.
However, 3D curves with fixed 3D locations cannot convey the geometric structure completely.
For example, the contour lines of smooth surfaces are critical to illustrating the overall shapes in ordinary sketches, but their surface locations change according to the viewpoints.

To embrace different structural elements, we separately model view-dependent and view-independent components:
\begin{equation}    \mathcal{S}^{\text{3D}}=\mathcal{S}_{\text{ind}}^{\text{3D}}\cup\mathcal{S}_{\text{dep}}^{\text{3D}}.
\end{equation}

The view-independent components $\mathcal{S}_{\text{ind}}^{3\text{D}}$ are represented as a set of 3D cubic  B\'ezier curves:
\begin{equation}
    \mathcal{S}_{\text{ind}}^{3\text{D}} = \left\{ B^{\text{3D}}(p_i) \right\}_{i=1}^{N_{\text{ind}}}, \quad p_i = (p_i^0, p_i^1, p_i^2, p_i^3)
\end{equation}
where $N_{\text{ind}}$ is number of strokes.
Each cubic B\'ezier curve $B^{\text{3D}}(p_i)$ is parameterized with the four ordered control points $p_i^j\in \mathbb{R}^3$:
\begin{equation}
    B^{\text{3D}}(t;p_i)=\sum_{j=0}^3 b_{j}(t)p_i^j, \quad b_{j}(t)={\binom{3}{j}}t^j(1-t)^{3-j},
\end{equation}
where $t\in[0, 1]$.
The positions of $4N_{\text{ind}}$ control points are optimized to generate illustrative 3D strokes.

Additionally, view-dependent components $\mathcal{S}_{\text{dep}}^{3\text{D}}$ encapsulate the 3D volume of the given object.
Inspired by the previous works~\cite{pentland1986parts, Alaniz_2023_ICCV, Paschalidou2019CVPR, Paschalidou2020CVPR}, we use the composition of superquadradics~\cite{barr1981superquadrics} as a compact parametric representation that can express various 3D shapes.
More concretely, 
\begin{equation}
    \mathcal{S}_{\text{dep}}^{3\text{D}}(\mathbf{d})=\sigma_{\text{contour}}\left(\bigcup_{i=1}^{N_{\text{dep}}} S(\theta_i), \mathbf{d}\right),
\end{equation}
where $N_{\text{dep}}$ is the number of superquadrics, and $\sigma_{\text{contour}}(G, \mathbf{d})$ is the view-dependent volume density function that generates the contour of any geometric volume $G$ at the given viewing direction $\mathbf{d}$.
We will discuss contour extraction in~\cref{subsec:diff_rendering}.

Each superquadric $S(\theta_i)$ is expressed with its parameter $\theta_i=\{\alpha_{i}, \epsilon_{i}, \mathbf{R}_i, \mathbf{t}_i\}$.
The first two parameters are from the implicit representation $f$ of the superquadric in its canonical coordinate
\begin{equation}
    \label{eq:implicit_superquadric}
    f(\mathbf{x};\alpha_i, \epsilon_i)=\left(\left(\frac{x}{\alpha_{i,1}}\right)^{\frac{2}{\epsilon_{i,2}}} + \left(\frac{y}{\alpha_{i,2}}\right)^{\frac{2}{\epsilon_{i,2}}}\right)^{\frac{\epsilon_{i,2}}{\epsilon_{i,1}}}+\left(\frac{z}{\alpha_{i,3}}\right)^{\frac{2}{\epsilon_{i,1}}},
\end{equation}
where $\epsilon_i=(\epsilon_{i,1},\epsilon_{i,2})$ decides the shape and $\alpha_i=(\alpha_{i,1}, \alpha_{i,2}, \alpha_{i,3})$ decides the scale along axes.
$\mathbf{x}$ lies on the surface of superquadric if $f=1$, outside if $f>1$, and inside if $f<1$.
$\mathbf{R}_i$ and $\mathbf{t}_i$ represent the rigid transformation between the world coordinate and $i$th superquadric's canonical coordinate:
\begin{equation}
    \label{eq:superquadric_transform}
    S(\mathbf{x};\theta_i)=f\left(\mathbf{R}_i^{-1}(\mathbf{x} - \mathbf{t}_i); \alpha_i, \epsilon_i\right).
\end{equation}

We can find the union of multiple superquadrics simply with  the minimum value of all individual implicit surface functions:
\begin{equation}
    \label{eq:superquadric_merge}
    S(\mathbf{x};\theta) = \bigcup_{i=1}^{N_{\text{dep}}} S(\mathbf{x}; \theta_i)=\min\limits_i S(\mathbf{x};\theta_i).
\end{equation}

\subsection{Differentiable Rendering of Sketch Lines}
\label{subsec:diff_rendering}

Given the 3D geometric primitives, we can draw black sketch lines on a white canvas $\mathcal{R}(\mathcal{S}^{\text{3D}})$.
A complete sketch is the union of view-independent and view-dependent components.
Similar to differentiable rendering of parametric objects~\cite{10.1145/3618387}, we propose a fully differentiable rendering pipeline to directly optimize the 3D parameters to produce abstract sketch images of the target objects.

\paragraph{View-Independent Sketch}

View-independent sketch lines involve rendering 3D cubic B\'ezier curves $B^{\text{3D}}$.
When we transform the 3D B\'ezier curve into 2D curves with perspective projection, the 2D curve is no longer a simple cubic spline but a 2D rational B\'ezier curve defined by projected control points.
Instead, we approximate the perspective projection to an orthographic projection, assuming that the camera is sufficiently far from the object and the perspective distortion is minimal.
We further discuss the effect of approximation in Appendix A.
Using orthographic projection, the following theorem holds:
\begin{restatable}[]{thm}{orthographic}
    \label{theo:orthographic}
    Orthographic projection of 3D B\'ezier curve $B^{\text{3D}}$ 
 on the image plane $\tilde{B}^{\text{2D}}$ is identical to the 2D B\'ezier curve $B^{\text{2D}}$ which is a cubic B\'ezier curve defined by $(q^0, q^1, q^2, q^3)$, where $q^j$ is an orthographic projection of 3D control point $p^j$ of $B^{\text{3D}}$.
\end{restatable}
\noindent We provide the proof of Theorem~\ref{theo:orthographic} in Appendix A.
Following Theorem~\ref{theo:orthographic}, we can render the 3D B\'ezier curve $B^{\text{3D}}$ at a given viewpoint by 1) perspective-projecting four control points and then 2) rendering a 2D B\'ezier curve with existing differentiable rasterizer~\cite{Li:2020:DVG}.
The rendering pipeline is fully differentiable since the projection operation and rasterizer are both differentiable.

\begin{figure}
    \centering
    \includegraphics[width=0.9\linewidth]{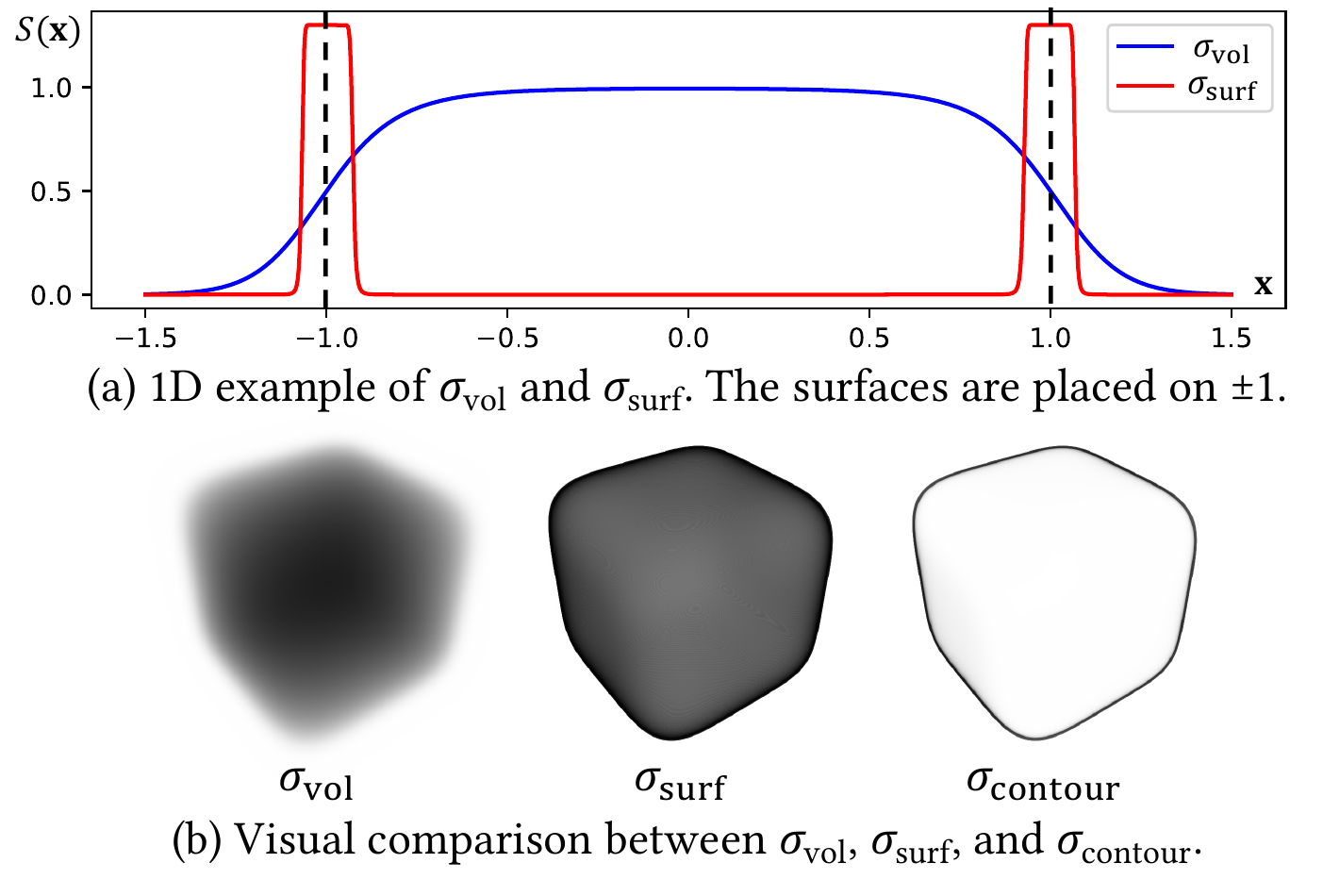}
    \caption{(a) We visualize the volume density $\sigma_{\text{vol}}$ and surface volume density $\sigma_{\text{surf}}$. (b) We display the volume-rendered results of each volume density component. One can obtain the contour of geometric volume by volume rendering our proposed view-dependent contour volume density $\sigma_{\text{contour}}$.}
    \label{fig:sigma_comparison}
\end{figure}

\paragraph{View-Dependent Sketch}
We provide the formulation to generate contour sketches that are differentiable to the superquadric parameters.
Volume rendering exploits smooth transition in volume density and provides a powerful tool to propagate image-space gradients to 3D geometry, as demonstrated by neural radiance fields~\cite{mildenhall2021nerf}.
Similarly, we propose a differentiable pipeline by defining smooth density fields representing sketch lines.
To allow continuous variation, ISCO~\cite{Alaniz_2023_ICCV} proposes volume density of superquadric
\begin{equation}
    \sigma_{\text{vol}}(\mathbf{x})=\text{sigmoid}(\gamma(1-S(\mathbf{x})),
\end{equation}
where $S$ is an implicit surface function defined in~\cref{eq:implicit_superquadric,eq:superquadric_transform,eq:superquadric_merge}.
$\gamma$ controls the slope of the shape boundary.
For sketch generation, we revise the volume density function to be concentrated near the surface $S=1$:
\begin{equation}
    \sigma_{\text{surf}}(\mathbf{x})=a\cdot \text{sigmoid}\left(\frac{1}{\gamma(1-S(\mathbf{x}))^2+\epsilon}-\gamma(1-S(\mathbf{x}))^2-b\right),
\end{equation}
where $a, b\in \mathbb{R}^+$ are hyperparameters determining the intensity and thickness of the surface volume.
We illustrate the $\sigma_{\text{vol}}$ and $\sigma_{\text{surf}}$ in one-dimensional example in~\cref{fig:sigma_comparison} (a).

The contours of sketch lines are known to be points whose normals are perpendicular to the viewing rays. 
We define view-dependent volume density that represents the contours of 3D superquadradics by attenuating the surface volume density function $\sigma_{\text{surf}}$ by the angle between $\mathbf{n}$ and $\mathbf{d}$:
\begin{equation}
    \sigma_{\text{contour}}(\mathbf{x}, \mathbf{d})=(1-(\mathbf{n}(\mathbf{x})\cdot \mathbf{d})^\beta)\sigma_{\text{surf}},
\end{equation}
where $\mathbf{d}$ is a viewing direction and $\mathbf{n}$ is a normal direction at the querying point $\mathbf{x}$.
Note that the normal direction $\mathbf{n}(\mathbf{x})=\left[\frac{\partial S}{\partial x}, \frac{\partial S}{\partial y}, \frac{\partial S}{\partial z}\right]$ can be calculated in a closed form with analytic derivative.
Then we render the view-dependent sketch $\mathcal{R}(\mathcal{S}_{\text{dep}}^{\text{{3D}}})$ using volume rendering technique and numerical quadrature approximation~\cite{max1995optical} with our contour volume density $\sigma_{\text{contour}}$.

\subsection{Optimization}
\label{subsec:optimization}

Given the differentiable pipeline, we optimize the parameters of our sparse 3D primitives to generate sketch images.
While there exists a significant domain gap between photo-realistic input images and the abstract sketch~\cite{vinker2022clipasso}, we design a loss function to encourage a balance between structural components and semantic perception compared to input multi-view images.
Specifically, we use LPIPS loss~\cite{zhang2018unreasonable} to capture rough geometric layout and CLIP score~\cite{radford2021learning} to maintain the high-level semantic meaning:
\begin{equation}
\begin{split}
    \mathcal{L}=\sum_{I\in \mathcal{I}}\lambda ~ \rho\left(\text{LPIPS}(I, \mathcal{R}(\mathcal{S}^{\text{3D}})), \alpha, c\right) \\+\, \text{dist}(\text{CLIP}(I), \text{CLIP}(\mathcal{R}(\mathcal{S}^{\text{3D}}))),
    \label{eq:loss}
\end{split}
\end{equation}
where $I$ is the image sampled from the input multi-view images $\mathcal{I}$, and $\text{dist}(a,b)=1-\frac{a\cdot b}{\|a\|\cdot\|b\|}$ is a cosine distance.
$\rho(x,\alpha,c)$ is a robust loss function~\cite{barron2019general}, which stabilize the optimization despite  outliers.
Outliers are inevitable as it is not trivial to correctly account for view-dependent occlusion only with line primitives.
We use $\alpha=1, c=0.1$ for all experiments.
More studies regarding the effects of loss terms can be found in the~\cref{sec:experiments}.
Our sketch parameters are directly optimized with Adam optimizer~\cite{kingma2014adam} by minimizing the loss term in~\cref{eq:loss}.

Due to the highly non-convex nature of our minimizing objective, appropriate initialization leads to fast and robust optimization.
We leverage the SfM (structure-from-motion) point clouds obtained while estimating the camera poses from input multi-view images.
We sample SfM points using farthest point sampling and initialize them to be locations of 3D B\'ezier curve control points and the center of superquadrics.
We further discuss the effects of the initialization method in~\cref{sec:experiments}.

\section{Experiments}
\label{sec:experiments}
\paragraph{Dataset \& Implementation Details}
The inputs to 3Doodle are multi-view images and the corresponding camera poses.
Our inputs are identical to recent neural scene reconstruction methods, and we evaluate the performance on the synthetic dataset from NeRF~\cite{mildenhall2021nerf} and InvRender~\cite{zhang2022invrender}.
We additionally collect diverse shapes of synthetic objects from open-source 3D models.
We render input images with Blender's Cycles path tracer, whose views are posed randomly on the upper hemisphere, centered at the object.
We include samples from our dataset in Appendix B.
We also test our method on the real-world scenes from the CO3D dataset~\cite{reizenstein2021common}.

Our code is mainly implemented with the auto differentiation library in PyTorch~\cite{paszke2019pytorch}.
We optimize the sketch parameters using Adam optimizer~\cite{kingma2014adam} with the default parameters and a learning rate of \num{1e-3}.
We render the 2D B\'ezier curves with a differentiable vector graphics library~\cite{Li:2020:DVG}, and superquadric contours with volume rendering from the vanilla NeRF model~\cite{mildenhall2021nerf}.
Our perceptual losses in Equation~(\ref{eq:loss}) employ pretrained RN101 model of CLIP encoder~\cite{radford2021learning} for the CLIP loss, and VGG16 model of LPIPS~\cite{zhang2018unreasonable} for the LPIPS loss.
Users additionally provide the number of B\'ezier curves and superquadrics to decide their desired abstraction level.
More details regarding our implementation can be found in Appendix.

\paragraph{Baselines}
To the best of our knowledge, we are the first to generate 3D sketches from multi-view images.
We compare the various characteristics of generated sketches against five baselines that employ different input/output representations.

Two of the baselines create 2D sketches from a single image.
Kampelm\"uhler and Pinz~\shortcite{kampelmuhler2020synthesizing} train a class-specific generative network from sketch-image pairs, and works only on specific object categories.
We therefore compare them only in balloon, car, chair, and teddy bear scenes.
CLIPasso~\cite{vinker2022clipasso} spawns a composition of sparse 2D strokes from image input.
We provide a single rendering of our 3D dataset for these baselines.

Three other baselines use 3D representations to draw sketch images with additional inputs.
Among them, we categorize Artistic Radiance Fields (ARF)~\cite{zhang2022arf} and suggestive contours~\cite{decarlo2003suggestive} to ones using detailed geometry to create dense lines.
ARF first trains NeRF and converts it into the style of a reference image, to which we feed the result of our sketch.
Suggestive contours extract geometric features from 3D mesh.
Given the input images only, we reconstruct the mesh from state-of-the-art neural SDF reconstruction~\cite{wu2022voxurf}.
Lastly, Neural Edge Fields (NEF)~\cite{Ye_2023_CVPR} output sparse 3D primitives similar to our approach. 
It produces geometric feature curves in 3D given multi-view 2D edge maps, estimated from PiDiNet~\cite{su2021pixel}.
However, NEF fails to reconstruct meaningful edge fields in our toycar and toyhorse scenes so we did not include the results.

\subsection{Quantitative Evaluation}
\begin{table}
    \centering
    \caption{Sketch recognition accuracy. We mark the best and second-best results in bold and underlined numbers. }
    \label{tab:sketch_recognition}
    \resizebox{\linewidth}{!}{
    \begin{tabular}{l|cccccc@{\:}}
    \toprule
    &Sketch& LPIPS$(\downarrow)$ & DINO$(\uparrow)$ & $\text{CLIP}^{\text{img}}$$(\uparrow)$& $\text{CLIP}^{\text{txt}}$$(\uparrow)$& Size \\
    \midrule
    ARF & \multirow{2}{*}{Dense}&\textbf{0.179} & \textbf{0.830} & 0.887 & 0.651 & $<3.8$GB\\
    Sugg. Contours& & 0.255 & \underline{0.828} & \textbf{0.897} & \underline{0.659} & $<70$MB\\
    \midrule
    CLIPasso &\multirow{3}{*}{Sparse}& 0.238 & 0.738 & 0.826 & 0.618 & $<1$kB\\
    NEF& & 0.324 & 0.721 & 0.852 & 0.592 & $\sim 6$MB\\
    3Doodle (ours) && \underline{0.217} & 0.784 & \underline{0.895} & \textbf{0.665} & $<1.5$kB\\
    \bottomrule
    \end{tabular}
    }
\end{table}

\begin{figure}
    \centering
    \includegraphics[width=\linewidth]{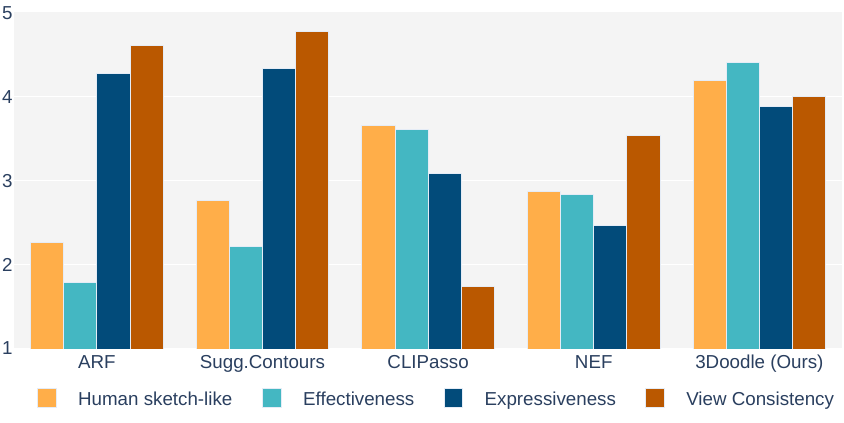}
    \caption{Bar plots of the perceptual study results.}
    \label{fig:user_study}
\end{figure}

As we value a hand-drawn sketch to be a subjective and perceptive mean of communication, it is not trivial to define a universal metric to indicate the quality of the output.
Instead, we devise several indirect measures that implicate the desirable characteristics of sketches. 
Specifically, we argue that 3Doodle can create human-like sketches that capture overall semantics in images but still maintain coherent yet sparse 3D structures.

\Cref{tab:sketch_recognition} utilizes deep features trained with large image database to make an assessment.
LPIPS~\cite{zhang2018unreasonable}, DINO~\cite{oquab2023dinov2}, and CLIP$^\text{img}$ compare the similarity between the generated sketch and the Blender rendering from novel viewpoints, demonstrating the effectiveness in delivering the structural (LPIPS, DINO) or abstract semantics (CLIP$^\text{img}$).
CLIP$^\text{txt}$ compares the similarity between the CLIP embedding vectors of sketches and text prompts defined as ``A sketch of a(n) \{\textit{object name}\}''.
We expect that the value implies how likely the generated images are perceived as conventional sketches.
Note that we use AlexNet~\cite{krizhevsky2012imagenet} and Vit-B/32~\cite{dosovitskiy2020image} models for evaluating LPIPS and CLIP scores, respectively, which are not used in our optimization process.
We normalize the cosine similarity values of embedding vectors from $[-1,1]$ to be $[0,1]$.
We additionally indicate the sparsity with the actual size of the generated sketch representation.

Our measure suggests that 3Doodle outperforms across all metrics against sketch generation methods with sparse representation (NEF and CLIPasso).
Dense representations (ARF and suggestive contours) preserves the structural similarity, as expected, and our approach shows slightly lower performance.
It is a reasonable result since LPIPS and DINO are trained with conventional images exhibiting dense pixel information.
However, for $\text{CLIP}^{\text{img}}$ similarity, which may be biased toward high-level semantics, 3Doodle achieves a score comparable to suggestive contours.
Especially, our approach marks the highest score against all baselines in $\text{CLIP}^{\text{txt}}$.
The results show that 3Doodle best conveys the semantic meaning of objects, and it is best recognized as a ``sketch''.
We also provide the memory requirement of the final sketch representations.
The dense representations require significantly large storage.
NEF also requires several orders of magnitude large memory size, as the edge information is stored in field representation.
CLIPasso and 3Doodle both maintain sparse primitives, achieving a significantly smaller memory footprint.
Our 3D primitives, however, achieve better performance on all evaluation metrics compared to 2D strokes in CLIPasso.

We further design a questionnaire to directly evaluate the perceptual implication of generated sketches.
We ask four key characteristics of the sketch that we are targeting, namely sketch-likeness, effectiveness, expressiveness, and view consistency.
The participants are asked to rate the sketches on a scale of 1 to 5 on the following questions: (i) How much do the images look like free-hand sketches drawn by humans?; (ii) how efficiently do the images represent the target object?; (iii) how well do the sketches represent the essential structure of the target object?; and (iv) how well do the sketches of moving views maintain structural consistency?
More details can be found in the Appendix.

The answers from 72 participants are summarized in~\cref{fig:user_study}.
3Doodle receives significantly higher score for being realistic sketches compared to the baselines using dense representations (ARF and suggestive contours).
The results also indicate that the participants appreciate the effectiveness of our sparse representation. 
Dense representations, on the other hand, can preserve the fine 3D geometry and maintain highly expressive and consistent structure, as expected.
3Doodle achieves a comparable score in expressiveness despite its compactness, which is a great improvement from other sparse representations (NEF and CLIPasso).
CLIPasso marks the lowest score in multi-view consistency, while 3Doodle maintains inherent view consistency that comes from the 3D representation.

\begin{figure}
    \centering
    \includegraphics[width=\linewidth]{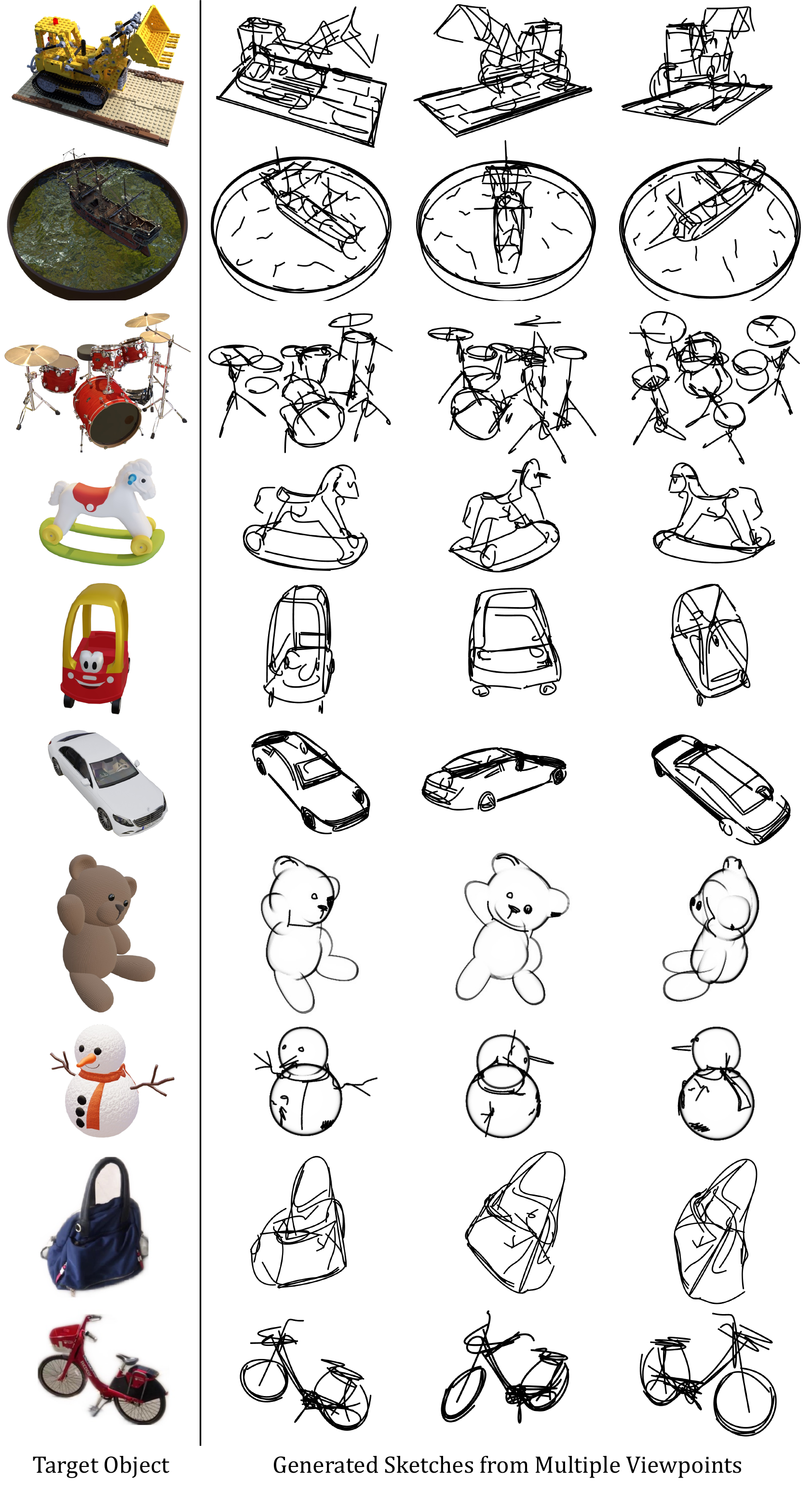}
    \caption{Qualitative results. We show the target objects in the leftmost column and multi-view rendered results of our 3D strokes in the right three columns.}
    \label{fig:qualitative}
\end{figure}

\subsection{Qualitative Evaluation}
We demonstrate the qualitative results of our sketches in~\cref{fig:teaser,fig:qualitative}.
3Doodle successfully expresses the essential structures of various types of objects with a highly compact representation.
Our view-independent sketches $\mathcal{S}^{\text{3D}}_\text{ind}$ capture both geometric features (e.g. fine structures of sails in the boat scene) and semantic features (e.g. a wave pattern in the boat scene).
Also, our view-dependent components $\mathcal{S}^{\text{3D}}_\text{dep}$ successfully represent the smooth bounding surfaces that envelop the objects.
While $\mathcal{S}^{\text{3D}}_\text{dep}$ expresses the outline of objects (e.g. shapes of teddy bear and snowman), $\mathcal{S}^{\text{3D}}_\text{ind}$ expresses the fine details from texture and sharp geometry (e.g. arms of snowman and face of teddy bear) to draw whole sketches together.

We provide a visual comparison between baselines with 3D representation in~\cref{fig:comparison_3d_methods}.
The sketches generated by suggestive contours are close to shading rather than human free-hand line drawing since the surfaces of reconstructed 3D meshes are not perfect as human assets.
Also, ARF, which converts the color of pre-trained plenoxels~\cite{fridovich2022plenoxels}, draws random strokes on the surface of objects (e.g. the random strokes on the plate of the hotdog scene and the surface of the teddy bear).
In contrast, NEF and 3Doodle produce more human-like drawings since both methods represent sketches as a set of geometric primitives.
However, NEF fails to represent the objects which contain smooth surfaces (teddy bear and snowman) since their edge fields only extract the view-independent parametric curves.
In contrast to NEF, 3Doodle successfully represents the sketch of the smooth surfaces with $\mathcal{S}^\text{3D}_\text{dep}$.

We further visually compare our approach with image sketch generation algorithms in~\cref{fig:comparison_2d_methods}.
The approach by Kampelm\"uhler and Pinz are not designed to generate sketches for unseen categories (e.g. lego scene) since they need training from sketch-image paired datasets.
Also, the sketches generated for the seen categories are not faithful to the target objects.
The optimization-base methods CLIPasso and 3Doodle successfully produce sketches that abstract a wide variety of objects with a sparse set of strokes.
However, CLIPasso fails to maintain consistency across the different viewpoints. (e.g. the vertical lines of the air balloon, lines on the chair seat, and the base of the Lego)
On the contrary, 3Doodle renders view-consistent 2D sketches due to the 3D representation.

\begin{figure}
    \centering
    \includegraphics[width=\linewidth]{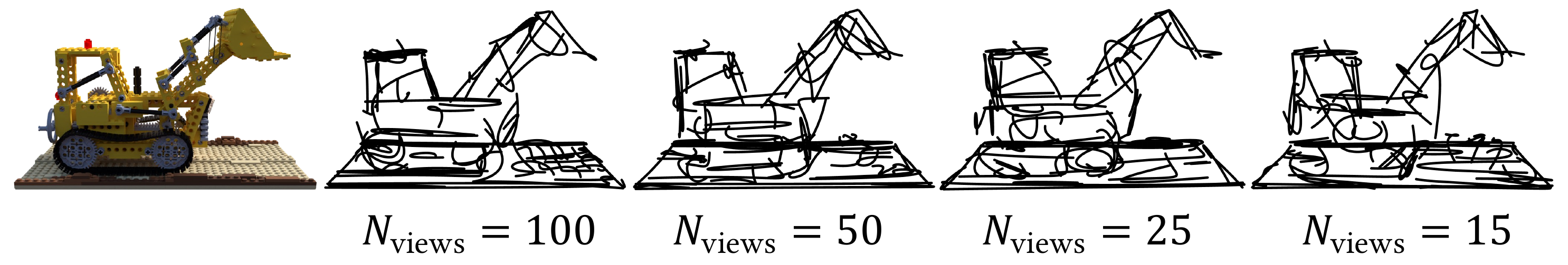}
    \caption{3Doodle robustly generates sketch for a few multi-view inputs.}
    \label{fig:few_view}
\end{figure}
\begin{figure}
    \centering
    \includegraphics[width=0.9\linewidth]{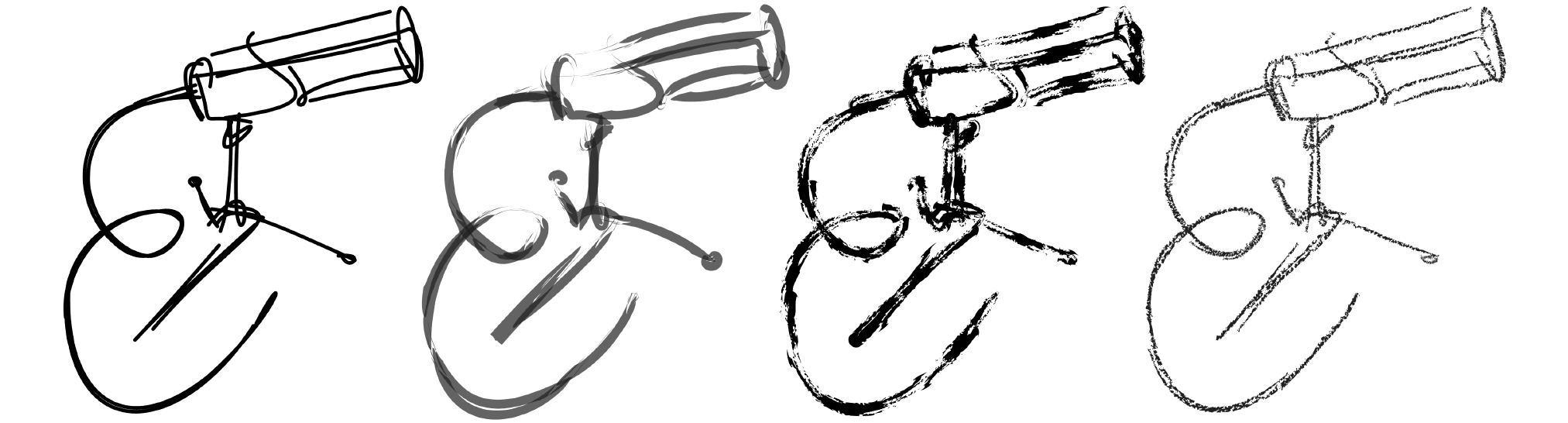}
    \caption{Sketch stylization. We demonstrate various styles of our generated sketch by rendering $\mathcal{S}^{\text{3D}}$ as vector graphics format and applying various brush styles from Adobe Illustrator.}
    \label{fig:svg_stylize}
\end{figure}
Furthermore, our 3D sketch have a small number of parameters, providing an additional advantage of robustly generating sketches even from a small number of input images.
In~\cref{fig:few_view}, we demonstrate the results of our sketch generated from a smaller number of input views.
We also uniformly sample the viewpoints on the upper hemisphere.
We observe that we can obtain a reasonable sketch from only 15 views which is much smaller than the default 100 views of the NeRF training dataset.
Our approach can also generate sketches from real-world captures in CO3D~\cite{reizenstein2021common} without further adjustments as shown in the last two columns in~\cref{fig:qualitative}.

Also, one can alter the style of our sketch by applying different brushes to the vector strokes as illustrated in~\cref{fig:svg_stylize}.

\subsection{Ablation Studies}
\begin{figure}
    \centering
    \includegraphics[width=0.8\linewidth]{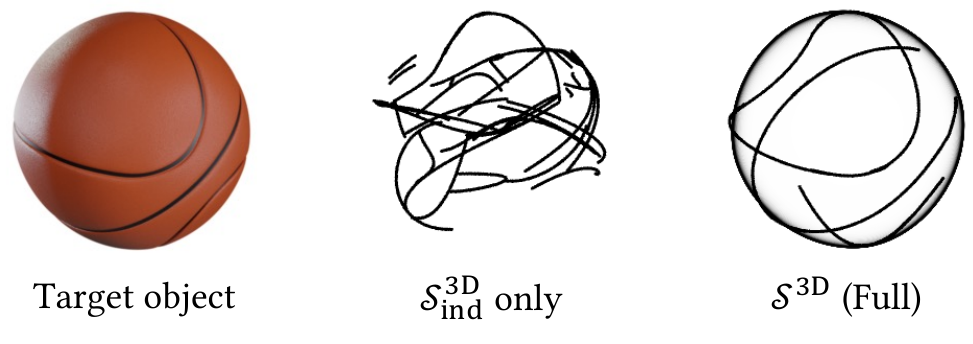}
    \caption{Effects of sketch components. The view-independent sketches cannot solely represent the objects that have smooth surfaces.}
    \label{fig:sketch_component}
\end{figure}

In this section, we analyze the effects of important components of 3Doodle.
First, we study about the necessity of the view-dependent component $\mathcal{S}^{\text{3D}}_{\text{dep}}$, which is our novel addition compared to previous works employing line-only representations.
The volumetric encapsulation is critical especially for sphere objects, as in~\cref{fig:sketch_component}.
When we only use 3D line strokes ($\mathcal{S}^{\text{3D}}_{\text{ind}}$) the 3D primitives clearly fail to express the contours of smooth surfaces.
NEF also uses view-independent 3D strokes, and exhibits similar characteristics when processing round objects with view-dependent contours, such as teddy bear and snowman in~\cref{fig:comparison_3d_methods}.

\begin{figure}
    \centering
    \includegraphics[width=\linewidth]{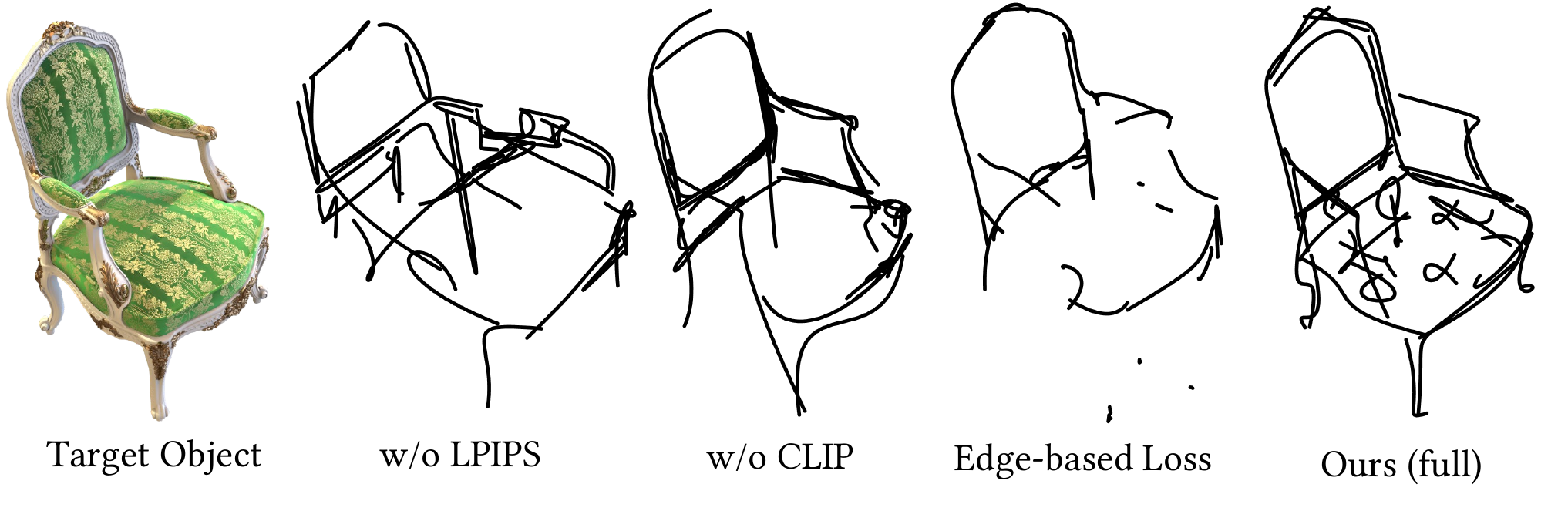}
    \caption{Effects of loss function.}
    \label{fig:loss_ablation}
\end{figure}

\begin{table}
    \centering
    \caption{Quantitative results on ablated loss terms}
    \label{tab:ablation}
    \resizebox{0.85\linewidth}{!}{
    \begin{tabular}{@{\:}l|cccc@{\:}}
    \toprule
    \textsc{Chair} scene& LPIPS$(\downarrow)$ & DINO$(\uparrow)$ & $\text{CLIP}^{\text{img}}$$(\uparrow)$& $\text{CLIP}^{\text{txt}}$$(\uparrow)$ \\
    \midrule
    w/o LPIPS & 0.351 & 0.763 & 0.879 & 0.649 \\
    w/o CLIP & 0.227 & 0.755 & 0.885 & 0.675 \\
    Edge-based loss & 0.276 & 0.740 & 0.867 & 0.646 \\
    3Doodle (full) & 0.216 & 0.772 & 0.897 & 0.660 \\
    \bottomrule
    \end{tabular}
    }
\end{table}

Then we evaluate the individual terms in our loss term.
As shown in~\cref{fig:loss_ablation}, both LPIPS and CLIP losses are crucial to effectively capture perceptual characteristics with sparse primitives.
Without LPIPS, the resulting sketch deviates from the given geometric structure.
On the other hand, removing CLIP loss retains the overall layout of the input object, but misses essential parts to be a `chair'.
We measure our quantitative metrics of the ablated versions with the chair object in~\cref{tab:ablation}.
Interestingly, removing CLIP loss does not necessarily achieve better LPIPS loss, and vice versa.
Our full loss function can find the better optimal point.
We also test a devised loss function which optimizes line strokes to match edge maps extracted from images~\cite{su2021pixel}.
All evaluation metrics are inferior, and the resulting sketch shows remaining short and un-optimized strokes.
We can infer that the loss directly employing detailed edge-based images suffers from a complex convergence basin with local minima.
Our full loss term efficiently finds an abstract sketch of an object.

\begin{figure}
    \centering
    \includegraphics[width=\linewidth]{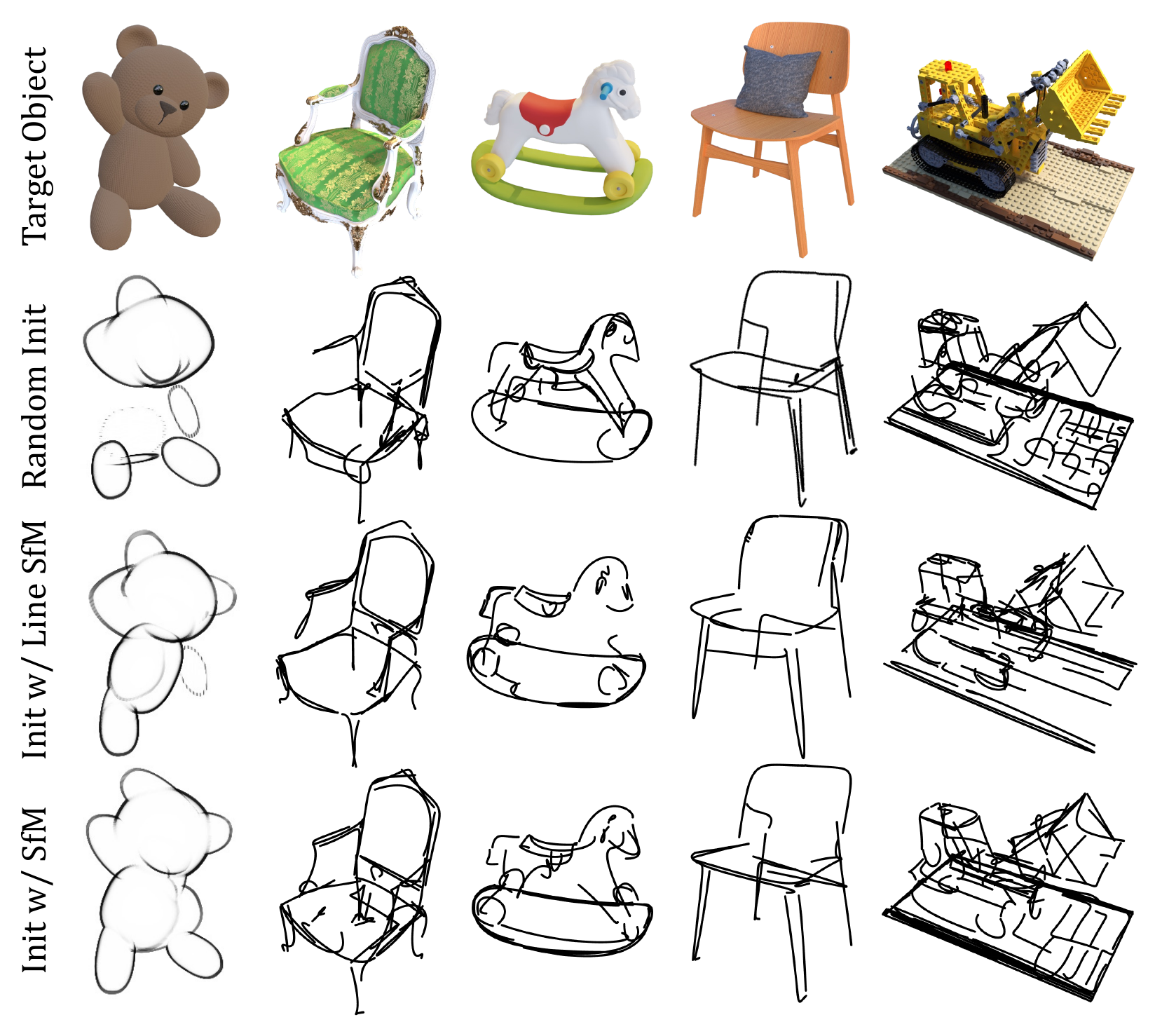}
    \caption{Effects of initialization step.}
    \label{fig:initialization}
\end{figure}

Next, we demonstrate the effectiveness of initializing our 3D primitives from SfM points in~\cref{fig:initialization}.
As a comparison, we show results where the control points and the superquadric centers are initialized randomly or using the output of line-based SfM~\cite{liu2023limap}.
Our approach is stable and produces a similar quality of results in most scenes.
However, our non-convex optimization relies on image observations and may condense 3D strokes in high-contrast regions or miss details without much image evidence, such as a white arm of a chair against a white background.
In such challenging scenarios, SfM point clouds provide a proper initialization to fully cover the shape of the target objects with a minimal number of primitives, such as superquadrics in the teddy bear.
Initialization from line-based SfM can partially mitigate the problem, especially for structures that can be approximated by straight lines, and sometimes results in faster convergence if we only use B\'ezier curves $\mathcal{S}^{\text{3D}}_{\text{ind}}$.
However, it struggles to reconstruct curvy boundaries in a teddy bear or a rocking pony. 
The details on the initialization pipeline can be found in Appendix.

\begin{figure}
    \centering
    \includegraphics[width=\linewidth]{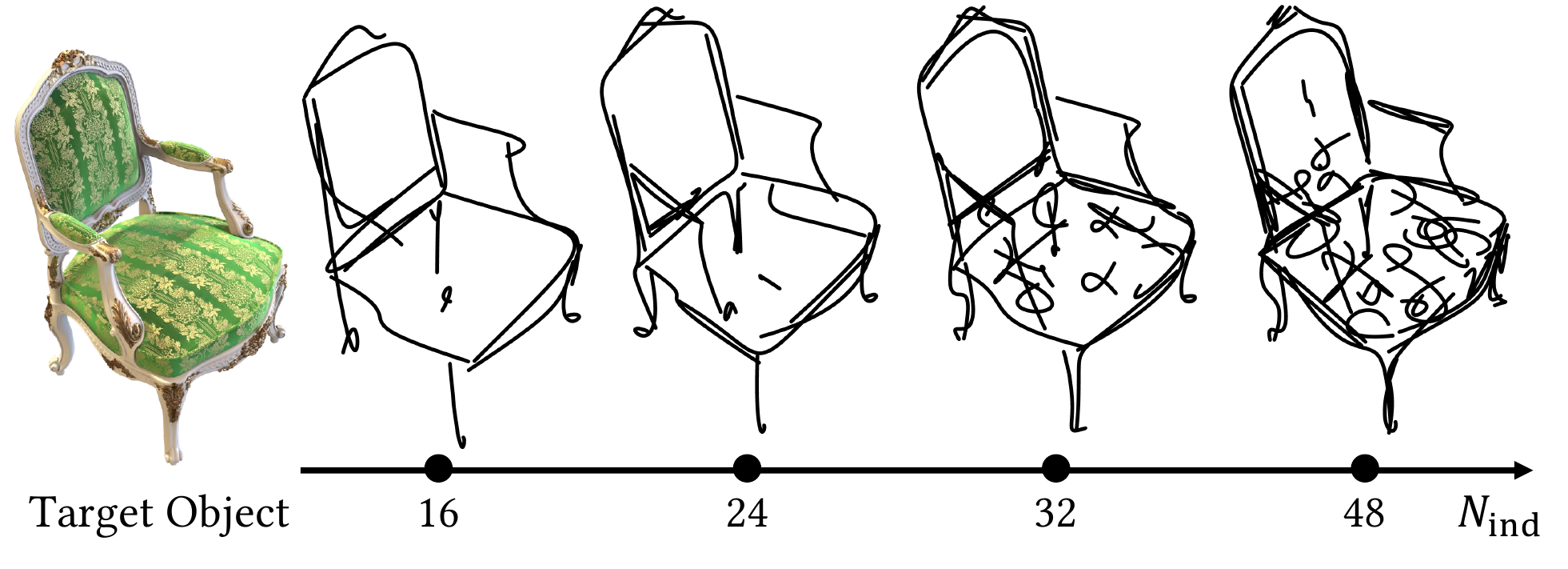} 
    \caption{The effects of using different numbers of sketch components $N_\text{ind}$. With more strokes, 3Doodle captures finer details. Fewer strokes lead to more abstract expression.}
    \label{fig:ablation_num_lines}
\end{figure}

We also study the effects of the number of strokes in 3Doodle.
In~\cref{fig:ablation_num_lines}, we display the different abstraction levels of sketches resulted by varying the number of B\'ezier curves.
Overall, 3Doodle can generate high-quality sketches with strokes ranging from 16 to 48 strokes.
As the number of strokes increases, 3Doodle effectively captures finer details, whereas a higher level of abstraction is achieved in expressing the object with fewer strokes.
However, when the number of B\'ezier curves is too small compared to the complexity of the input object, 3Doodle may fail to represent the geometric details of the target object (e.g., the right-side armrest of the chair is erased using only 16 strokes in~\cref{fig:ablation_num_lines}).

\section{Conclusion and Future Works}
\label{sec:conclusion}
We propose 3Doodle, the sketch generation pipeline from 3D primitives that are optimized from multi-view images. 
The suggested view-dependent and view-independent components compose a highly effective and compact set of 3D structural descriptions.
Equipped with our novel differentiable formulation, we optimize a small number of parameters with perceptual losses in the image domain instead of the precise geometric formulation.
Our results can be rendered in arbitrary viewpoints and successfully deliver essential semantics of various objects while mimicking the abstract styles of human sketches.

3Doodle has many interesting potential applications.
As an automatic way to generate multiple sketch images of the 3D shape, 3Doodle can build a dataset of 3D sketch-images or 3D sketch-shape pairs, which is more scalable than having humans draw the sketches~\cite{luo2021fine}.
The fully differentiable pipeline can be easily extended to creative content creations beyond multi-view renderings,  such as text-to-3D sketch generation or 3D sketch video generation.
Furthermore, the compact structural proxy may assist other practical applications that require 3D geometric information.
Future works include using the approximate 3D shape to guide recognition, localization~\cite{liu2023limap,fabbri20103d}, few-shot NeRF~\cite{deng2022depth,wang2023sparsenerf}, or other 3D reconstruction tasks, where geometric or depth priors are known to accelerate and stabilize the outcome.

\paragraph{Limitations}
While 3Doodle can generate expressive sketches, the number of primitives is fixed as input parameters. 
It would be an interesting future work to automatically find the optimal number of strokes that balance the expressivity and sparsity, which may vary significantly depending on the subject and the applications.
In addition, the optimization of parameters requires a decent amount of computation time (at most 6 hours), although the final representation is compact and quick to render.
Lastly, our current formulation ignores the depth ordering between geometric primitives and visualizes the entire wireframe without occlusion.
Because we already have the volumetric representation of superquadradic, we can use its volume to account for occlusion and skip rendering textures on the other side of the volume (e.g., the faces of the snowman or teddy bear can be hidden from the rear view).

\begin{acks}
This work was supported by the NRF grants (No.RS-2023-00208197 (50\%) and No.2023R1A1C200781211(30\%)), IITP grant funded by the Korea government(MSIT) (No.2021-0-01343, AI Graduate School Program(Seoul National University)(10\%)) funded by the Korea government(MSIT), INMC, and the BK21 FOUR program of the Education and Research Program for Future ICT Pioneers, Seoul National University in 2024(10\%). Young Min Kim is corresponding author.
\end{acks}

\bibliographystyle{ACM-Reference-Format}
\bibliography{references}

\clearpage
\begin{figure*}
    \centering
    \includegraphics[width=0.98\linewidth]{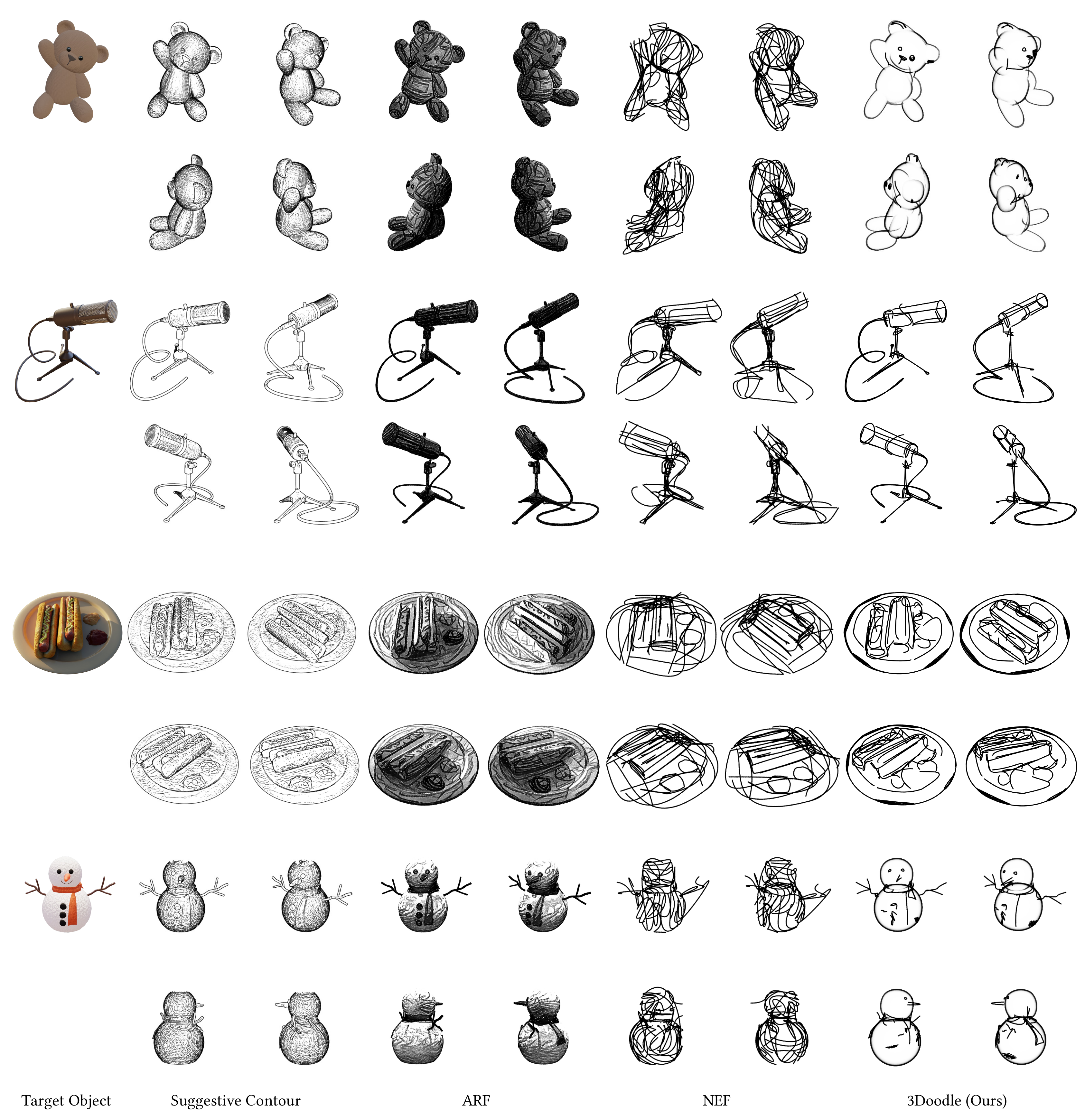}
    \caption{Qualitative comparison with baseline methods which have 3D representations.}
    \label{fig:comparison_3d_methods}
\end{figure*}

\begin{figure}
    \centering
    \includegraphics[width=0.98\linewidth]{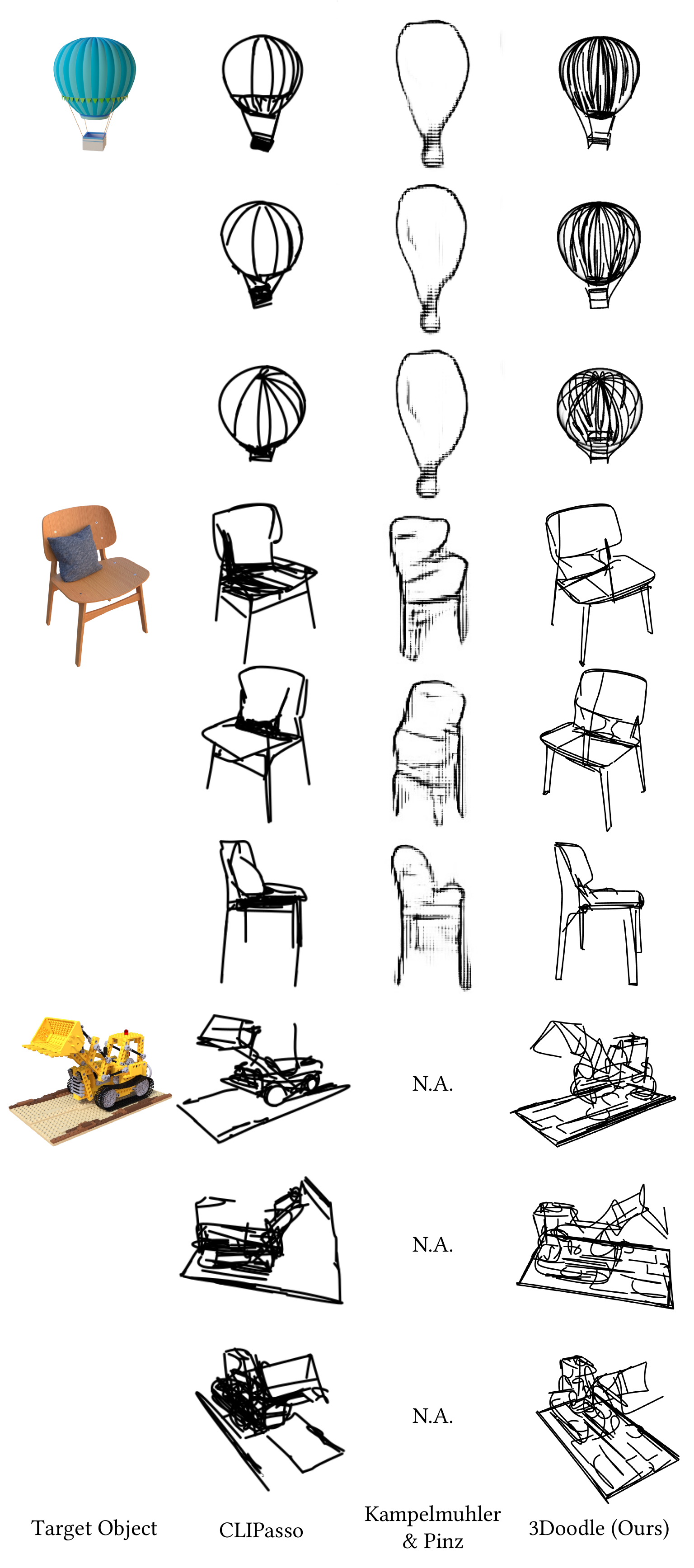}
    \caption{Qualitative comparison with image sketch generation methods.}
    \label{fig:comparison_2d_methods}
\end{figure}

\clearpage
\appendix
\section{Projection of 3D B\'ezier curve}
In this section, we provide a mathematical proof of 2D B\'ezier curve parameterized with orthogonally projected control points is identical to the orthogonally projected 3D B\'ezier curve in section~\ref{subsec:orthogonal_proj_proof} and we explain the approximation we use in perspective projection in section~\ref{subsec:perspective_proj_approx}.

\subsection{Orthogonal Projection}
\label{subsec:orthogonal_proj_proof}
We provide proof of~\cref{theo:orthographic} in this subsection.
\orthographic*
\renewcommand\qedsymbol{$\blacksquare$}
\begin{proof}
Without loss of generality, we can assume that the camera is looking at $z$ direction and the image plane is $xy$ plane at $z=0$.
Then, the projection of 3D B\'ezier curve $B^{\text{3D}}(t)=\sum_{j=0}^3 b_j(t)p^j$ on the $xy$ plane $\tilde{B}^{\text{2D}}$ is
\begin{equation}
    \tilde{B}^{\text{2D}}(t)=\sum_{j=0}^3b_j(t)\begin{pmatrix}p_x^j\\p_y^j\\0\end{pmatrix},
\end{equation}
where the control points of $B^{\text{3D}}$, $p^j=\begin{pmatrix}p_x^j\\p_y^j\\p_z^j\end{pmatrix}$ since $b_j(t)$ is scalar and the projection of any 3D points onto $xy$ plane is equal to simply removing the $z$ component.
As $q^j=\begin{pmatrix}p_x^j\\p_y^j\\0\end{pmatrix}$, $\tilde{B}^{\text{2D}}$ is identical to the 2D B\'ezier curve represented by the projected 3D control points of $B^{\text{3D}}$.
\end{proof}

\subsection{Perspective Projection}
In this subsection, we provide proof of~\cref{theo:perspective}, which describes the property of perspective projection of 3D cubic B\'ezier curve.
Then, we explain when our approximation is reasonable.
\label{subsec:perspective_proj_approx}
\begin{restatable}[]{thm}{perspective}
    \label{theo:perspective}
    Perspective projection of 3D B\'ezier curve $B^{\text{3D}}$ 
 on the image plane $\tilde{B}^{\text{2D}}$ is identical to the 2D rational B\'ezier curve $B^{\text{2D}}$ which is a cubic rational B\'ezier curve defined by $(q^0, q^1, q^2, q^3)$, where $q^j$ is a perspective projection of 3D control point $p^j$ of $B^{\text{3D}}$ and the weight associated to each control points is the depth from the camera center.
\end{restatable}
\begin{proof}
Without loss of generality, we can assume that the camera is looking at $z$ direction and the image plane is $z=f$ ($f$ is focal length and the image plane is parallel to $xy$ plane).
Then, the perspective projection of 3D B\'ezier curve $B^{\text{3D}}(t)=\sum_{j=0}^3b_j(t)p^j$ on the image plane $z=f$ can be written as follows: 

\begin{align}
    \tilde{B}^{\text{2D}}(t)&=\begin{pmatrix}B^{\text{3D}}_x(t)\dfrac{f}{B^{\text{3D}}_z(t)}\\B^{\text{3D}}_y(t)\dfrac{f}{B^{\text{3D}}_z(t)}\end{pmatrix}
\end{align}
\begin{align}
    \phantom{\tilde{B}^{\text{2D}}(t)}&=\begin{pmatrix}
        \dfrac{\sum_{j=0}^3b_j(t)fp_x^j}{\sum_{j=0}^3b_j(t)p_z^j}\\
        \dfrac{\sum_{j=0}^3b_j(t)fp_y^j}{\sum_{j=0}^3b_j(t)p_z^j}
    \end{pmatrix}\\
    &=\begin{pmatrix}
        \dfrac{\sum_{j=0}^3b_j(t)\left(\frac{f}{p_z^j}p_x^j\right)p_z^j}{\sum_{j=0}^3b_j(t)p_z^j}\\
        \dfrac{\sum_{j=0}^3b_j(t)\left(\frac{f}{p_z^j}p_y^j\right)p_z^j}{\sum_{j=0}^3b_j(t)p_z^j}
    \end{pmatrix}\\
    &=\begin{pmatrix}
        \dfrac{\sum_{j=0}^3b_j(t)q_x^j\omega_j}{\sum_{j=0}^3b_j(t)\omega_j}\\
        \dfrac{\sum_{j=0}^3b_j(t)q_y^j\omega_j}{\sum_{j=0}^3b_j(t)\omega_j}
    \end{pmatrix},
\end{align}
when we set $\omega_j=p_z^j$.
Thus, $\tilde{B}^{\text{2D}}(t)$, which is the perspective projection of 3D cubic B\'ezier curve $B^{\text{3D}}$ is a rational B\'ezier curve defined by the perspective projected control points $q^j$ and the assigned weight is the depth from the camera center $p_z^j$.
\end{proof}

Following~\cref{theo:perspective}, our B\'ezier curve approximation of the rational B\'ezier curve is valid when the weights assigned to four control points are nearly the same.
Namely, our approximation is reasonable when the viewpoints are located sufficiently far from the interested objects.

\section{Other Details}
\subsection{Dataset}
\begin{figure}
    \centering
    \includegraphics[width=\linewidth]{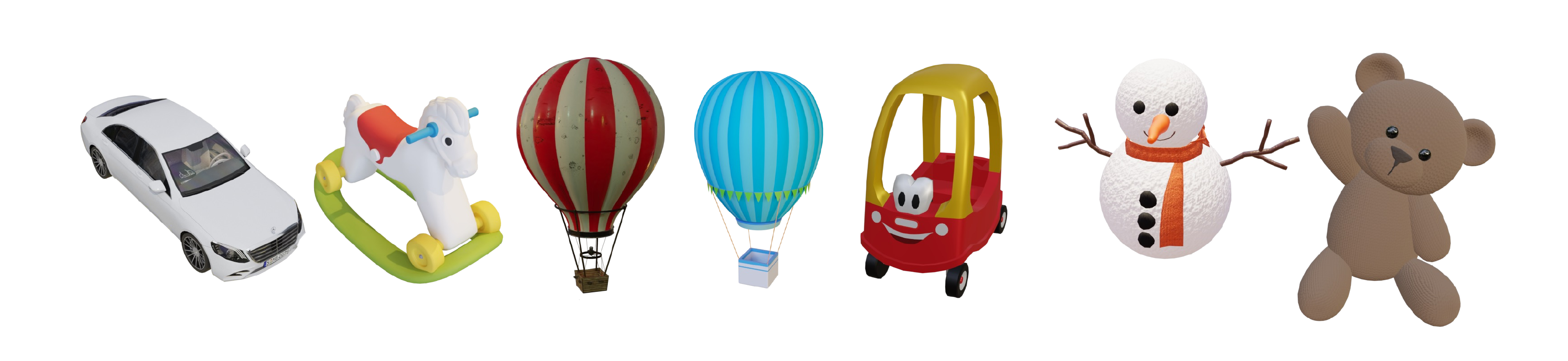}
    \caption{Samples from our new dataset}
    \label{fig:dataset}
\end{figure}
We test the performance of 3Doodle in [\textsc{Chair, Drums, Hotdog, Lego, Mic, Ship}] scenes in the synthetic blender dataset from Neural Radiance Fields~\cite{mildenhall2021nerf}, and [\textsc{Chair}] scene from Invrender~\cite{zhang2022invrender}.
We additionally collected synthetic scenes to evaluate 3Doodle in various object shapes.
We show the samples from our dataset in~\Cref{fig:dataset}.
We slightly modify the publicly available 3D models from various sources below:
\definecolor{codegreen}{rgb}{0,0.6,0}
\definecolor{backcolour}{rgb}{0.96,0.96,0.96}

\lstdefinestyle{egonerf}{
backgroundcolor=\color{backcolour},   
    commentstyle=\color{codegreen}\itshape,
    basicstyle=\ttfamily\scriptsize,
    breakatwhitespace=false,         
    breaklines=true,                 
    captionpos=b,                    
    keepspaces=true,                 
    numbers=none,                    
    numbersep=5pt,                  
    showspaces=false,                
    showstringspaces=false,
    showtabs=false,                  
    tabsize=2,
    emph={Basketball, Benz, Blueballoon, Redballoon, Snowman, Teddybear, Toycar, Toyhorse},
    emphstyle=\bfseries\color{codegreen},
    keywordstyle=none
}
\lstset{style=egonerf}
\begin{lstlisting}[language=python]
- Basketball by animatedheaven (CC-0)
https://blendswap.com/blend/25086
- Benz by PieEntertainment (No AI)
https://www.cgtrader.com/free-3d-models/car/luxury-car/mercedes-benz-s65-amg-w222
- Blueballoon by Mokokow (CC-BY)
https://www.blendswap.com/blend/31301
- Redballoon by Psyhycal (CC-0)
https://www.blendswap.com/blend/27824
- Snowman by Gugurun (no AI)
https://www.cgtrader.com/free-3d-models/exterior/other/free-snowman-christmas-and-new-years-thank-gift
- Teddybear by hectopod (CC-BY)
https://sketchfab.com/3d-models/teddy-bears-e84b12b4ac20402aaf4d40f2219cd0e2
- Toycar by ikbenanders (no AI)
https://www.cgtrader.com/free-3d-models/vehicle/other/childrens-rideable-toy-car
- Toyhorse by PoliVR (no AI)
https://www.cgtrader.com/free-3d-models/sports/toy/colorful-plastic-toy-horse-for-kids
\end{lstlisting}

\subsection{Perceptual Study}
We provide details of our perceptual studies.
We collected responses from randomly selected 72 participants.
The participants are asked to answer the three questions corresponding to the randomly selected five objects and their sketches from baselines and our method.
The asked question is:
\begin{enumerate}
    \item How much do the images look like \textbf{free-hand sketches} drawn by humans?
    \item On a scale from 1 to 5, rate how \textbf{efficiently} (with a minimum number of elements) represents the target object.
    \item On a scale from 1 to 5, rate how the sketch represents the \textbf{essential structure} of the target object.
\end{enumerate}
Then, we ask the participants to evaluate the performance of consistency after watching the videos of rotating sketches for five randomly selected objects.
The asked question is:
\begin{enumerate}
    \item On a scale of 1 to 5, please rate how well the corresponding outputs maintain \textbf{consistency} from multiple viewpoints.
\end{enumerate}

\subsection{Implementation Details}
In this section, we provide more details regarding rendering and optimizing our sketch parameters.
\paragraph{Rendering 3D Strokes}
To render the B\'ezier curve, we use the fixed width for the strokes in all of our experiments.
However, na\"ive volume-rendering our $\sigma_\text{contour}$ leads non-uniform stroke widths in the $\epsilon_i$-$\alpha_i$ space.
The small $\alpha_i$ and $\epsilon_i$ far from 1 makes the stroke thinner.
Therefore, we set the width-related hyperparameters in $\mathcal{S}^\text{3D}_\text{dep}$ to be adaptive with respect to $\alpha_i$ and $\epsilon_i$:
\begin{equation}
    \gamma_\text{adaptive} = \begin{cases} \min\{f(\gamma_\text{min}, \gamma_\text{max}, \alpha_i), f(\gamma_\text{min}, \gamma_\text{max}, \epsilon_i)\}& \epsilon_i <= 1\\
    \min\{f(\gamma_\text{min}, \gamma_\text{max}, \alpha_i), \gamma_\text{max}\} & \epsilon_i > 1
    \end{cases},
\end{equation}
where
\begin{equation}
    f(t_\text{min}, t_\text{max}, x) = t_\text{min} + (t_\text{max} - t_\text{min}) \dfrac{x - 0.1}{0.9}.
\end{equation}
Similarly, 
\begin{equation}
    b_\text{adaptive} = \begin{cases} \min\{f(b_\text{min}, b_\text{max}, \alpha_i), f(b_\text{min}, b_\text{max}, \epsilon_i)\}& \epsilon_i <= 1\\
    \min\{f(b_\text{min}, b_\text{max}, \alpha_i), b_\text{max}\} & \epsilon_i > 1
    \end{cases},
\end{equation}
\begin{equation}
    a_\text{adaptive} = \begin{cases} \max\{g(a_\text{min}, a_\text{max}, \alpha_i),\\
    \qquad g(a_\text{min}, a_\text{max}, \epsilon_i)\}& \alpha_i, \epsilon_i > 0.3\\
    \max\{ a_\text{min}, a_\text{max}, \alpha_i) \} & \alpha_i > 0.3, \epsilon_i <= 0.3\\
    \max\{g(a_\text{min}, a_\text{max}, \epsilon_i) , a_\text{min}\} & \epsilon > 0.3, \alpha <= 0.3\\
    a_\text{min} & \alpha_i, \epsilon_i <= 0.3,
    \end{cases},
\end{equation}
where
\begin{equation}
    g(t_\text{min}, t_\text{max}, x) = t_\text{max} - (t_\text{max} - t_\text{min}) \dfrac{x - 0.1}{0.2}.
\end{equation}
\paragraph{Optimization}
When we optimize the scenes represented by both view-dependent sketches and view-independent sketches, we first optimize the superquadric parameters and then we optimize the B\'ezier curve parameters.
Also, we do not apply robust loss for scenes that are represented by view-independent components solely.
We initialize B\'ezier curves by sampling control points for the pre-defind number of stroke primitives.
The control points for the real-world scenes and the teddy bear are randomly sampled from a bounding box.
For all other objects, we use farthest point sampling on the SfM point cloud.
The centers of superquadrics are also initialized by applying the farthest point sampling on the SfM points.

\paragraph{Line SfM-based Initialization}
\begin{figure}
    \centering
    \includegraphics[width=\linewidth]{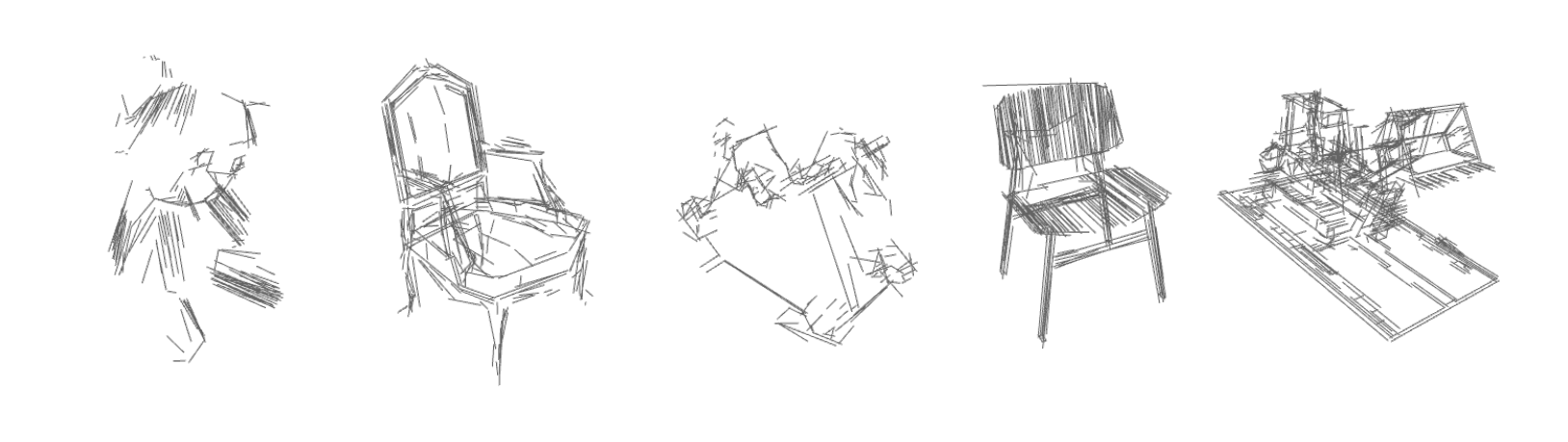}
    \caption{Reconstructed line maps for initialization.}
    \label{fig:linemap}
\end{figure}
In this subsection, we describe the details of our experiments of initialization method with line-based SfM.
First, we obtain the 3D line maps from the known camera poses by triangulating the detected and matched line segments using LIMAP library~\cite{liu2023limap}.
We show the reconstructed line maps in~\cref{fig:linemap}.
The line-based SfM method can reconstruct the 3D line maps containing the structural details when the view-consistent edges are detectable (e.g., two chairs and lego).
However, it struggles to reconstruct valid 3D line maps when the objects contain smooth contours (teddy bear and rocking pony), and consequently it leads to suboptimal results as demonstrated in~\cref{fig:initialization} of the main manuscript.
To initialize our sketch parameters with the linemaps, we first sample the line segments using the farthest sampling method based on the distance between lines.
After lines are sampled, we use the centers of lines as the initial center of superquadrics. We also extract random samples along the lines to initialize the four control points of B\`ezier curves.

\end{document}